\documentclass[journal]{IEEEtran}
\usepackage{amsmath}
\usepackage{amsthm}
\usepackage{amssymb}
\usepackage{graphicx}
\usepackage{color}
\usepackage{enumerate}
\usepackage{multirow}

\newtheorem{theorem}{Theorem}
\newtheorem{corollary}{Corollary}
\newtheorem{lemma}{Lemma}
\newtheorem{definition}{Definition}
\newtheorem{conjecture}{Conjecture}

\newcommand{\R}{\mathbb{R}}

\DeclareMathOperator*{\argmax}{arg\,max}

\usepackage{cite}

\ifCLASSINFOpdf
\else
\fi
\begin{document}
%
\title{Learning Poisson systems and trajectories of autonomous systems via Poisson neural networks}


\author{\IEEEauthorblockN{Pengzhan Jin,
Zhen Zhang,
Ioannis G. Kevrekidis, and
George Em Karniadakis}

\thanks{The work was supported in part by the DOE PhILMs Project under Grant DE-SC0019453 and in part by the AFOSR MURI Project under Grant FA9550-20-1-0358. The work of I.G.K. was partially supported by DARPA and the Army Research Office. (Corresponding author: George Em Karniadakis.)

Pengzhan Jin is with LSEC, ICMSEC, Academy of Mathematics and Systems Science, Chinese Academy of Sciences, Beijing 100190, China (e-mail: jpz@lsec.cc.ac.cn).

Zhen Zhang and George Em Karniadakis are with the Division of Applied Mathematics, Brown University, Providence, RI 02912 USA (e-mail: zhen\_zhang1@brown.edu; george\_karniadakis@brown.edu).

Ioannis G. Kevrekidis is with the Department of Chemical and Biomolecular Engineering, Johns Hopkins University, Baltimore, Maryland 21211, USA (e-mail: yannisk@jhu.edu).
}
}

%



\IEEEtitleabstractindextext{%
\begin{abstract}
We propose the Poisson neural networks (PNNs) to learn Poisson systems and trajectories of autonomous systems from data. Based on the Darboux-Lie theorem, the phase flow of a Poisson system can be written as the composition of (1) a coordinate transformation, (2) an extended symplectic map and (3) the inverse of the transformation. In this work, we extend this result to the unknotted trajectories of autonomous systems. We employ structured neural networks with physical priors to approximate the three aforementioned maps. We demonstrate through several simulations that PNNs are capable of handling very accurately several challenging tasks, including the motion of a particle in the electromagnetic potential, the nonlinear Schr{\"o}dinger equation, and pixel observations of the two-body problem.
\end{abstract}

\begin{IEEEkeywords}
SympNets, geometric learning, physics-informed neural networks, Darboux-Lie theorem.
\end{IEEEkeywords}
}

\maketitle

\IEEEdisplaynontitleabstractindextext

%
\IEEEpeerreviewmaketitle

\section{Introduction}
%
%
%
%
\IEEEPARstart{T}{he} connection between dynamical systems and neural network models has been widely studied in the literature, see, for example,  \cite{chang2017reversible,weinan2017proposal,haber2017stable,lu2018beyond,lu2020mean}. In general, neural networks can be considered as discrete dynamical systems with the basic dynamics at each step being a linear transformation followed by a component-wise nonlinear (activation) function. In \cite{chen2018neural} the neural ODE is introduced as a continuous-depth model instead of specifying a discrete sequence of hidden layers. Even before the introduction of neural ODEs, a series of models with similar architectures had already been proposed to learn the hidden dynamics of a dynamical system in \cite{rico1994continuous,yannis1998rk,raissi2018multistep}. Theoretical results of the discovery of dynamics are established and enriched in \cite{zhu2020inverse}, where the inverse modified differential equations are introduced to understand the true dynamical system that is learned when the time derivatives are approximated by numerical schemes. The methods above do not assume the specific form of the equation a priori, however, as the physical systems usually possess intrinsic prior properties, other approaches also take into account the prior information of the systems.

For many physical systems from classical mechanics, the governing equation can be expressed in terms of Hamilton's equation. We denote the $d$-by-$d$ identity matrix by $I_{d}$, and let
\begin{equation*}
J:=\begin{pmatrix} 0 & I_{d} \\ -I_{d} & 0 \end{pmatrix},
\end{equation*}
which is an orthogonal, skew-symmetric real matrix, so that $J^{-1}=J^{T}=-J$. The canonical Hamiltonian system can be written as
\begin{equation} \label{eq:ham_eq}
\dot{y}=J^{-1}\nabla H(y),
\end{equation}
where $y(t)\in \R^{2d}$, and $H:\R^{2d}\to\R$ is the Hamiltonian, typically representing the energy of the system. It is well known that the phase flow of a Hamiltonian system is symplectic. Based on that observation, several numerical schemes which preserve symplecticity have been proposed to solve the forward problem in \cite{feng1984difference, hairer2006geometric, lubich2008quantum}. In recent works \cite{tom2019ham,greydanus2019hamiltonian,rezende2019equivariant,sanchez2019hamiltonian,chen2020symplectic,Toth2020Hamiltonian}, the primary focus has been to solve the inverse problem, i.e., identifying Hamiltonian systems from data, using structured neural networks. For example, HNNs \cite{greydanus2019hamiltonian} use a neural network $\Tilde{H}$ to approximate the Hamiltonian $H$ in (\ref{eq:ham_eq}), then learn $\Tilde{H}$ by reformulating the loss function. Based on HNNs, other models were proposed to tackle problems in generative modeling \cite{rezende2019equivariant,Toth2020Hamiltonian} and continuous control \cite{Zhong2020Symplectic}. Another line of approach is to learn the phase flow of the system directly, while encoding the physical prior as symplecticity in the flow. Recently, we introduced symplectic networks (SympNets) with theoretical guarantees that they can approximate arbitrary symplectic maps \cite{jin2020sympnets}. Other networks of this type are presented in \cite{dipietro2020sparse,xiong2020nonseparable}.

In practice, requiring a dynamical system to be Hamiltonian could be too restrictive, as the system has to be described in canonical coordinates. In \cite{cranmer2020lagrangian} Lagrangian Neural Networks (LNNs) were introduced, which allow the system to be expressed in Cartesian coordinates. In \cite{finzi2020simplifying} the models of HNNs and LNNs were generalized to Constrained Hamiltonian Neural Networks (CHNNs) and Constrained Lagrangian Neural Networks (CLNNs), enabling them to learn constrained mechanical systems written in Cartesian coordinates. In other developments, autoencoder-based HNNs (AE-HNNs) \cite{greydanus2019hamiltonian} and Hamiltonian Generative Networks (HGNs) \cite{Toth2020Hamiltonian} were proposed to learn and predict the images of mechanical systems, which can be seen as Hamiltonian systems on manifolds embedded in high-dimensional spaces. Theoretically, Hamiltonian systems on manifolds written in noncanonical coordinates are equivalent to an important class of dynamical systems, namely, the Poisson systems. To wit, the Poisson systems take the form of
\begin{equation*}
    \dot{y}=B(y)\nabla H(y),
\end{equation*}
where $y\in \R^n$, $n$ is not necessarily an even number; $H$ is the Hamiltonian of the Poisson system, and the matrix-valued function $B(y)$ plays the role of $J^{-1}$ in (\ref{eq:ham_eq}), which induces a general Poisson bracket as defined in Section \ref{sec:prelim}. The Darboux-Lie theorem states that a Poisson system can be turned into a Hamiltonian system by a local coordinate transformation. As a consequence, structure-preserving numerical schemes for Poisson systems are normally developed by finding the coordinate transformation manually, then applying symplectic integrators on the transformed systems, see  \cite{tang1996symplectic, hairer2006geometric}.

Inspired by the Darboux-Lie theorem, we propose a novel neural network architecture, the Poisson neural network (PNN), to learn the phase flow of an arbitrary Poisson system, In other words, PNNs can learn any unknown diffeomorphism of Hamiltonian systems. The coordinate transformation and the phase flow of the transformed Hamiltonian system are parameterized by structured neural networks with physical priors. Specifically, in the general setting, we use invertible neural networks (INNs) \cite{dinh2014nice, dinh2016density} to represent the coordinate transformation. If all the data reside on a submanifold with dimension $2d<n$, autoencoders (AEs) can be applied to approximate the coordinate transformation from the coordinates in $\R^n$ to its local coordinates as an alternative choice. This strategy is similar to \cite{rico1992discrete}, which learns dynamics in the latent space discovered through an autoencoder. Compared to LNNs, PNNs are able to work on a more general coordinate system. Moreover, INN-based PNNs are able to learn multiple trajectories of a Poisson system on the whole data space simultaneously, while AE-HNNs, HGNs, CHNNs and CLNNs are only designed to work on low-dimensional submanifolds of $\R^{n}$. Further, our work lays a solid theoretical background for all the aforementioned models, suggesting that they are learning a Poisson system explicitly or implicitly.

Another intriguing property of PNNs is that they are not only capable of learning Poisson systems, but are able to approximate an unknotted trajectory of an arbitrary autonomous system. We present related theorems, which indicate the great expressivity of PNNs. We demonstrate through computational experiments  that PNNs can be practically useful in terms of learning high-dimensional autonomous systems, long-time prediction, as well as frame interpolation. PNNs also enjoy all the advantages listed in \cite{jin2020sympnets} as they are implemented based on the SympNets in this work. However, PNNs as a high-level architecture can also employ other modern symplectic neural networks.

The rest of the paper is organized as follows. Section \ref{sec:prelim} introduces some necessary notation, terminology and fundamental theorems that will be used. The learning theory for PNNs is presented in Section \ref{sec:theory}. Section \ref{sec:result} presents the experimental results for several Poisson systems and autonomous systems. A summary is given in the last section. Supporting materials, including the detailed implementation of PNNs, are included in the appendix.

\section{Preliminaries}\label{sec:prelim}
The material required for this work is based on the mathematical background of Hamiltonian system and its non-canonical form, i.e., the Poisson system. We refer the readers to \cite{hairer2006geometric} for more details.
\subsection{Hamiltonian and Poisson systems}
First, we formally present the definitions of the Hamiltonian system and the Poisson system. We assume that all the functions or maps involved in this paper are as smooth as needed.
\begin{definition}
The canonical Hamiltonian system takes the form
\begin{equation} \label{eq:ham_sys}
\dot{y}=J^{-1}\nabla H(y),
\end{equation}
where $H:U\to\R$ is the Hamiltonian typically representing the energy of the system defined on the open set $U\subset \R^{2d}$.
\end{definition}
System (\ref{eq:ham_sys}) can also be written in a general form by introducing the Poisson bracket.
\begin{definition} \label{def:poisson_brac}
Let $U\subset\R^{n}$ be an open set. 
The Poisson bracket $\{\cdot,\cdot\}:C^\infty(U)\times C^\infty(U)\to C^\infty(U)$ is a binary operation satisfying
\begin{enumerate}[(i)]
    \item (anticommutativity)\\
    $\{F,G\}=-\{G,F\}$,
    \item (bilinearity)\\
    $\{aF+bG,H\}=a\{F,H\}+b\{G,H\},\\ \{H,aF+bG\}=a\{H,F\}+b\{H,G\},\quad a,b\in\R$,
    \item (Leibniz's rule)\\
    $\{FG,H\}=\{F,H\}G+F\{G,H\}$,
    \item (Jacobi identity)\\
    $\{\{F,G\},H\}+\{\{H,F\},G\}+\{\{G,H\},F\}=0$,
\end{enumerate}
for $F,G,H\in C^\infty(U)$.
\end{definition}
Consider the bracket
\begin{equation} \label{eq:can_brac}
    \{F,G\}=\sum_{i=1}^{d}\left(\frac{\partial F}{\partial q_i}\frac{\partial G}{\partial p_i}-\frac{\partial F}{\partial p_i}\frac{\partial G}{\partial q_i}\right)=\nabla F(y)^TJ^{-1}\nabla G(y),
\end{equation}
where $y=(p,q)=(p_1,\cdots,p_d,q_1,\cdots,q_d)\in\R^{2d}$. One can check that (\ref{eq:can_brac}) is indeed a Poisson bracket. Then system (\ref{eq:ham_sys}) can be written as
\begin{equation*} 
    \dot{y}_i=\{y_i,H\},\quad i=1,\cdots,2d,
\end{equation*}
where $y_i$ in the bracket denotes the map $y\to y_i$ for $y=(y_1,\cdots,y_{2d})$ by a slight abuse of notation. Now we extend the bracket (\ref{eq:can_brac}) to a general form as
\begin{equation} \label{eq:gen_brac}
\begin{split}
    (\{F,G\}_B)(y)&=\sum_{i,j=1}^{n}\frac{\partial F(y)}{\partial y_i}b_{ij}(y)\frac{\partial G(y)}{\partial y_j} \\
    &=\nabla F(y)^TB(y)\nabla G(y),
\end{split}
\end{equation}
where $B(y)=(b_{ij}(y))_{n\times n}$ is a smooth matrix-valued function. Note that here we do not require $n$ to be an even number. As many crucial properties of Hamiltonian systems rely uniquely on the conditions $(i)$-$(iv)$ in Definition \ref{def:poisson_brac}, we naturally expect the bracket (\ref{eq:gen_brac}) to be a Poisson bracket.
\begin{lemma} \label{lem:gen_brac}
The bracket defined in (\ref{eq:gen_brac}) is anti-commutative, bilinear and satisfies Leibniz's rule as well as the Jacobi identity if and only if
\begin{equation*}
    b_{ij}(y)=-b_{ji}(y)\ for\ all\ i,j
\end{equation*}
and for all $i,j,k$
\begin{equation*}
    \sum_{l=1}^{n}\left(\frac{\partial b_{ij}(y)}{\partial y_l}b_{lk}(y)+\frac{\partial b_{jk}(y)}{\partial y_l}b_{li}(y)+\frac{\partial b_{ki}(y)}{\partial y_l}b_{lj}(y)\right)=0.
\end{equation*}
\end{lemma}
Lemma \ref{lem:gen_brac} provides verifiable equivalence conditions for (\ref{eq:gen_brac}) to become a Poisson bracket. Then, we can give the definition of Poisson system, which is actually the generalized form of the Hamiltonian system.
\begin{definition}
If a matrix-valued function $B(y)$ satisfies Lemma \ref{lem:gen_brac}, then (\ref{eq:gen_brac}) defines a Poisson bracket, and the corresponding differential system
\begin{equation*} 
\dot{y}=B(y)\nabla H(y)
\end{equation*}
is a Poisson system. Here $H$ is still called a Hamiltonian.
\end{definition}
Up to now, the Hamiltonian system and the Poisson system have been unified as
\begin{equation} \label{eq:unified_sys}
    \dot{y}_i=\{y_i,H\}_B,\quad i=1,\cdots,n,
\end{equation}
for $B$ satisfying Lemma \ref{lem:gen_brac}, and the system becomes Hamiltonian when $B=J^{-1}$.

\subsection{Symplectic map and Poisson map}
The study of the phase flows of the Hamiltonian and Poisson systems focuses on the symplectic map and the Poisson map.
\begin{definition} \label{def:symp_map}
A transformation $\Phi:U\to\R^{2d}$ (where $U$ is an open set in $\R^{2d}$) is called a symplectic map if its Jacobian matrix satisfies
\begin{equation*}
    \left(\frac{\partial \Phi}{\partial y}\right)^{T}J\left(\frac{\partial \Phi}{\partial y}\right)=J.
\end{equation*}
\end{definition}

\begin{definition} \label{def:pois_map}
A transformation $\Phi:U\to\R^{n}$ (where $U$ is an open set in $\R^{n}$) is called a Poisson map with respect to the Poisson system (\ref{eq:unified_sys}) if its Jacobian matrix satisfies
\begin{equation*}
    \left(\frac{\partial \Phi}{\partial y}\right)B(y)\left(\frac{\partial \Phi}{\partial y}\right)^T=B(\Phi(y)).
\end{equation*}
\end{definition}
In fact, the phase flow $\phi_t^H(y)$ of the Hamiltonian system is a symplectic map, while the phase flow $\phi_t^P(y)$ of the Poisson system is a Poisson map, i.e., $\phi_t^H(y)$ and $\phi_t^P(y)$ satisfy Definition \ref{def:symp_map} and \ref{def:pois_map}, respectively. Based on these facts, we naturally expect the numerical methods or learning models for Hamiltonian systems and Poisson systems to preserve the intrinsic properties that $\phi_t^H(y)$ and $\phi_t^P(y)$ possess. So far the numerical techniques for Hamiltonian systems have been well developed \cite{feng1984difference, hairer2006geometric, lubich2008quantum}, however, research on the Poisson systems is ongoing due to its complexity.

\subsection{Coordinate changes and the Darboux–Lie theorem}

The main idea in studying Poisson systems is to find the connection to Hamiltonian systems, which are easier to deal with. In fact, a Poisson system expressed in arbitrary new coordinates is again a Poisson system, hence we naturally tend to simplify a given Poisson structure as much as possible by coordinate transformation.
\begin{theorem}[Darboux 1882, Lie 1888] \label{thm:darboux}
Suppose that the matrix $B(y)$ defines
a Poisson bracket and is of constant rank $n-q=2d$ in a neighbourhood of $y_0\in\R^n$. Then, there exist functions $P_1(y),\cdots,P_d(y)$, $Q_1(y),\cdots,Q_d(y)$, and $C_1(y),\cdots,C_q(y)$ satisfying
\begin{equation*}
\begin{aligned}
  &\{P_i,P_j\}=0 &
  &\{P_i,Q_j\}=-\delta_{ij} &
  &\{P_i,C_l\}=0 \\
  &\{Q_i,P_j\}=\delta_{ij} &
  &\{Q_i,Q_j\}=0 &
  &\{Q_i,C_l\}=0 \\
  &\{C_k,P_j\}=0 &
  &\{C_k,Q_j\}=0 &
  &\{C_k,C_l\}=0 \\
\end{aligned}
\end{equation*}
on a neighbourhood of $y_0$, where $\delta_{ij}$ equals to 1 if $i=j$ else 0. The gradients of $P_i, Q_i,C_k$ are linearly independent, so that $y\to(P_i(y),Q_i(y),C_k(y))$ constitutes a local change of coordinates to canonical form.
\end{theorem}
\begin{corollary}[Transformation to canonical form] \label{cor:trans}
Let us denote the transformation of Theorem \ref{thm:darboux} by $z=\theta(y)=(P_i(y),Q_i(y),C_k(y))$. With this change of coordinates, the Poisson system $\dot{y}=B(y)\nabla H(y)$ becomes
\begin{equation*}
\dot{z}=B_0\nabla K(z)\quad with\quad B_0=\begin{pmatrix}J^{-1} & 0 \\ 0 & 0\end{pmatrix},
\end{equation*}
where $K(z)=H(y)$. Writing $z=(p,q,c)$, this system becomes
\begin{equation*}
\dot{p}=-K_q(p,q,c),\quad \dot{q}=K_p(p,q,c),\quad \dot{c}=0.
\end{equation*}
\end{corollary}
Corollary \ref{cor:trans} reveals the connection between Poisson systems and Hamiltonian systems via coordinate changes. In a forward problem, i.e., solving the Poisson system by numerical integration, transformations are available for many well-known systems to perform structure-preserving calculations \cite{tang1996symplectic, hairer2006geometric}, but there does not exist a general method to search for the new coordinates of an arbitrary Poisson system, which is still an open research issue. However, the inverse problem, i.e.,  learning an unknown Poisson system based on data, is an easier task.

\section{Learning theory for Poisson systems and trajectories of autonomous systems}
\label{sec:theory}
Assume that there is a dataset $\mathcal{T}=\{(x_i,y_i)\}_1^N$ ($x_i,y_i\in\R^n$) from an unknown autonomous dynamical system, that could be a Poisson system or not, satisfying $\phi_h(x_i)=y_i$ for time step $h$ and phase flow $\phi_t$. We aim to discover the dynamics using learning models, so that we can make predictions into future or perform some other computational tasks. To describe things clearly, we first give the definition of the extended symplectic map.
\begin{definition}
A transformation $\Phi:U\to\R^{n}$ (where $U$ is an open set in $\R^{n}$) is called an extended symplectic map with latent dimension $2d$ if it can be written as
\begin{equation*}
    \Phi\begin{pmatrix} x_1 \\ x_2 \end{pmatrix}=\begin{pmatrix} \phi(x_1,x_2) \\ x_2 \end{pmatrix},\quad x_1\in\R^{2d},x_2\in\R^{n-2d},2d\leq n, 
\end{equation*}
where $\phi$ is differentiable, and $\phi(\cdot,x_2)$ is a symplectic map for each fixed $x_2$. Note that $\Phi$ degenerates to a general symplectic map when $2d=n$.
\end{definition}
\subsection{Poisson neural networks}
We propose a high-level network architecture, i.e., the Poisson neural networks (PNNs), to learn Poisson systems or autonomous flows based on the Darboux–Lie theorem. Theorem \ref{thm:darboux} and Corollary \ref{cor:trans}
indicate that any Poisson system in $n$-dimensional space can be transformed to a ``piecewise'' Hamiltonian system, where $2d\leq n$ is the latent dimension determined by the rank of $B(y)$. The architecture is composed of three parts: (1) a transformation, (2) an extended symplectic map, (3) the inverse of the transformation, denoted by $\theta$, $\Phi$ and $\theta^{-1}$, respectively. For the construction of extended symplectic neural networks we refer the readers to Appendix \ref{impl:sympnet}, which is an important part of this work. The transformations $\theta$ and $\theta^{-1}$ can be implemented using two alternative approaches. \\
\textbf{Primary architecture.} We model $\theta$ as an invertible neural network, as in \cite{dinh2014nice, dinh2016density}, to automatically obtain its inverse $\theta^{-1}$. Then, we learn the data by optimizing the mean-squared-error loss
\begin{equation*}
    L(\mathcal{T}) = \frac{1}{n\cdot N}\sum_{i=1}^N\|\theta^{-1}\circ\Phi\circ\theta(x_i)-y_i\|^2,
\end{equation*}
where $\Phi:\R^n\to\R^n$ is an extended symplectic neural network with latent dimension $2d$. \\
\textbf{Alternative architecture.} We exploit an autoencoder to parameterize $\theta$, $\theta^{-1}$ with two different neural networks. Then, the loss is designed as
\begin{equation*}
\begin{split}
    &L(\mathcal{T})\\
    =&L_s(\mathcal{T}) + \lambda\cdot L_a(\mathcal{T}) \\
    =& \frac{1}{2d\cdot N}\sum_{i=1}^N\|\Phi\circ\theta(x_i)-\theta(y_i)\|^2+ \\
    &\lambda\cdot\frac{1}{n\cdot N}\sum_{i=1}^N(\|\theta^{-1}\circ\theta(x_i)-x_i\|^2+\|\theta^{-1}\circ\theta(y_i)-y_i\|^2),
\end{split}
\end{equation*}
where $\Phi:\R^{2d}\to\R^{2d}$ is a symplectic neural network and $\lambda$ is a hyperparameter to be tuned. Note that in this case $\theta^{-1}\circ\theta$ is not intrinsically equivalent to the identity map. Basically, this architecture only learns the Poisson map limited on a submanifold embedded in the whole phase space.

In both cases we perform predictions by $f_{PNN}^k=\theta^{-1}\circ\Phi^k\circ\theta$, which gives the $k$th step. We prefer the invertible neural networks because the reconstruction loss will disappear compared to the autoencoder, and we also expect to impose further prior information on the transformation. For example, in Section \ref{ssec:lorentz}, we adopt a volume-preserving network as the invertible neural network, the so-called volume-preserving Poisson neural network (VP-PNN), which achieves better generalization compared to the non-volume-preserving Poisson neural network (NVP-PNN), since the original Poisson system has a volume-preserving phase flow. More crucially, autoencoder-based PNNs are unable to learn data lying on the whole space, rather than a $2d$-dimensional submanifold, when $2d<n$. However, autoencoder-based PNNs can perform better in some situations, such as in the numerical case in Section \ref{ssec:tb}. Intuitively, the alternative architecture outperforms the primary architecture when $2d\ll n$. An illustration of PNNs is presented in Fig.~\ref{fig:architecture}.
\begin{figure}[ht]
    \centering
    \includegraphics[width=0.45\textwidth]{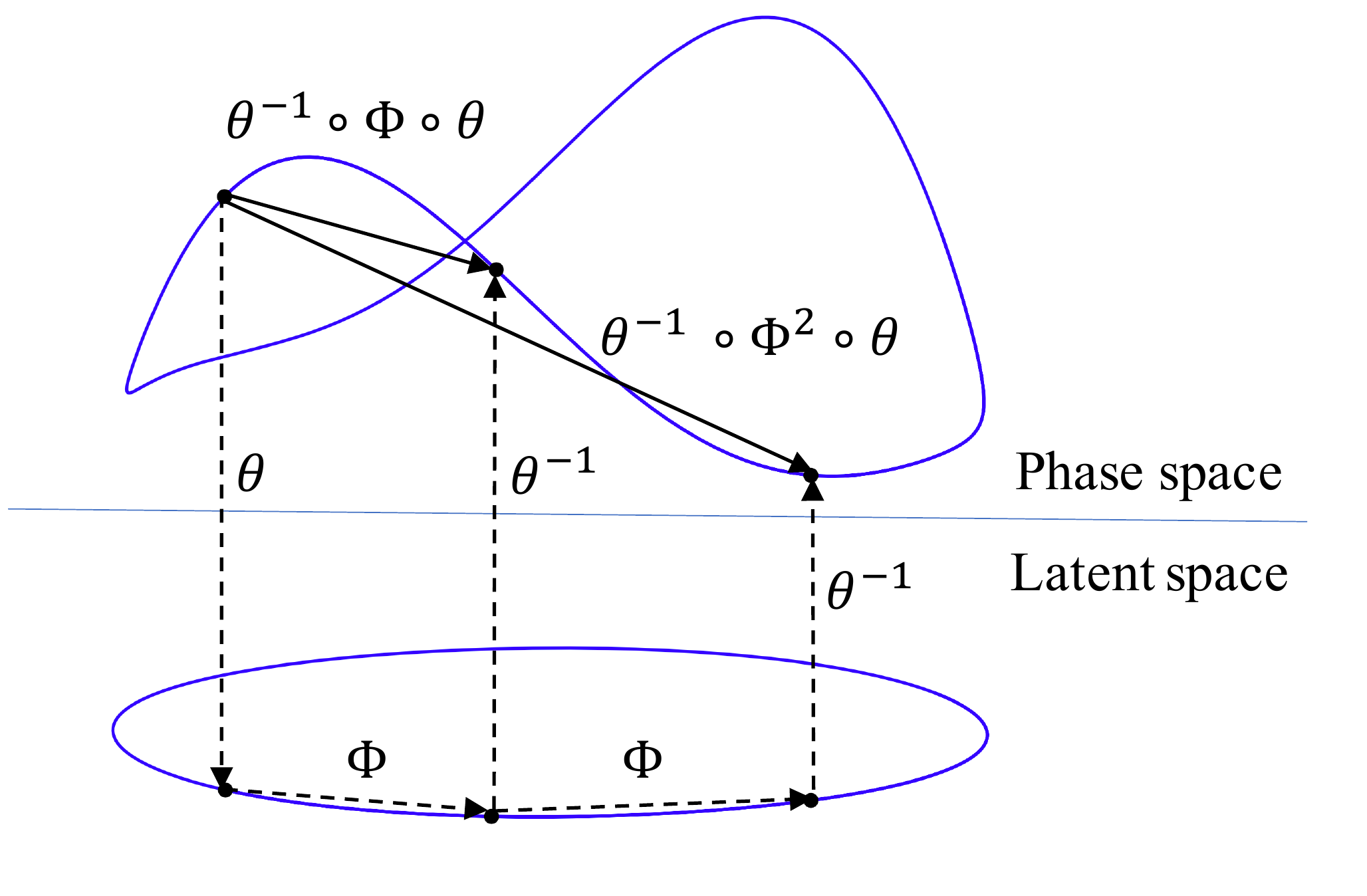}
    \caption{\textbf{Architecture of PNNs.} PNNs are composed of three parts: (1) a transformation, (2) an extended symplectic map, and (3) the inverse of the transformation, denoted by $\theta$, $\Phi$, and $\theta^{-1}$ respectively. The dimension of the latent space is equal to the original phase space, and the key difference is that the transformed system in latent space is ``piecewise'' Hamiltonian of latent dimension $2d$ while the whole space is of dimension $n\geq 2d$.}
    \label{fig:architecture}
\end{figure}
\subsection{Learning Poisson systems}
Consider the case that $\mathcal{T}$ consists of data points from a Poisson system. Next, we present the approximation properties of PNNs.
\begin{theorem}
Suppose that (i) the extended symplectic neural networks $\Phi$ are universal approximators within the space of extended symplectic maps in $C^1$ topology, (ii) (primary architecture) the invertible neural networks $\theta$ are universal approximators within the space of invertible differentiable maps in the $C^1$ topology, and (iii) (alternative architecture) the transformation neural networks $\theta$ and $\theta^{-1}$ are universal approximators within the space of continuous maps in the $C^0$ topology. Then, the corresponding Poisson neural networks $\theta^{-1}\circ\Phi\circ\theta$ are universal approximators within the space of Poisson maps in $C^1$ topology for the primary architecture, and are able to approximate arbitrary Poisson maps (limited on submanifolds) within the space of continuous maps in $C^0$ topology for the alternative architecture. The approximations are considered on compact sets.
\end{theorem}
\begin{proof}
It can be deduced directly from Theorem \ref{thm:darboux} and Corollary \ref{cor:trans}.
\end{proof}
Notice that in the alternative architecture, the PNN $\theta^{-1}\circ\Phi\circ\theta$ itself is not intrinsically a Poisson map, however, by using $f_{PNN}^k=\theta^{-1}\circ\Phi^k\circ\theta$ for multistep prediction, it can also preserve geometric structure and enjoy stable long term performance in practice.
\subsection{Learning trajectories of autonomous systems}
Now consider the case that $\mathcal{T}$ consists of a series of data points on a single trajectory (maybe not from a Poisson system), i.e, $\mathcal{T}=\{(x_{i-1},x_i)\}_1^N$, where $\phi_h(x_{i-1})=x_i$ for time step $h$ and $\phi_t$, which is the phase flow of an unknown autonomous system $\dot{y}=f(y)$. Unlike the theory for learning Poisson systems, the use of PNNs to learn autonomous flows is quite novel. The theories are driven by the observation that symplectic neural networks can learn a trajectory from a non-Hamiltonian system, see Section \ref{ssec:lv} for details. The next theorem reveals the internal mechanism.
\begin{theorem} \label{thm:2d_auto}
Suppose that $U\subset\R^2$ is a simply connected open set, and the periodic solution $y(t)\in U$ is from an autonomous dynamical system
\begin{equation*}
    \dot{y}=f(y),\quad y\in U,
\end{equation*}
then there exists a Hamiltonian $H(y)$, such that $y(t)$ also satisfies the Hamiltonian system
\begin{equation*}
    \dot{y}=J^{-1}\nabla H(y),\quad y\in U.
\end{equation*}
\end{theorem}
\begin{proof}
The proof can be found in Appendix \ref{proof:2d_auto}.
\end{proof}
Based on above theorem, one may be able to apply symplectic neural networks to arbitrary periodic solution to autonomous system in $\R^2$. Naturally, we tend to explore similar results in high-dimensional space. We intuitively expect to transform any high-dimensional periodic solution to autonomous system into one lying on a plane with the help of coordinate changes, and subsequently, the original trajectory can be learned via PNN. In fact, this conjecture is almost right, except for the case when the orbit of the considered motion is a non-trivial 1-knot in $\R^3$. For the theory on knots we refer to \cite{livingston1993knot, armstrong2013basic,ranicki2013high}, and we briefly present the basic concepts in Appendix \ref{sec:knot}.
\begin{theorem} \label{thm:nd_auto}
Suppose that $U\subset\R^n$ is a contractible open set, periodic solution $y(t)\in U$ is from an autonomous dynamical system
\begin{equation*}
    \dot{y}=f(y),\quad y\in U,
\end{equation*}
and the orbit of $y(t)$ is unknotted. Then, there exists a Hamiltonian $H(y)$ and a $B(y)$ satisfying Lemma \ref{lem:gen_brac} with rank of $2$, such that $y(t)$ also satisfies the Poisson system
\begin{equation*}
    \dot{y}=B(y)\nabla H(y),\quad y\in U.
\end{equation*}
Note that non-trivial $1$-knots exist only in $\R^3$.
\end{theorem}
\begin{proof}
The proof can be found in Appendix \ref{proof:nd_auto}.
\end{proof}
Up to now, we have shown that PNNs can be used to learn almost any periodic solution to autonomous systems, and the latent dimension is actually fixed as 2. The symplectic structure embedded in PNNs will endow the predictions with long term stability and more accuracy, for periodic solutions. Nevertheless, there is still a limitation of this method, as one can see, PNNs are allowed to learn only a single trajectory upon training. Basically, the limitation is inevitable as we have already got rid of most of the constraints on the vector field $f$, which is in fact a trade-off between data and systems. In spite of this fact, we still expect to further develop a strategy to relax the requirements of ``single'' and ``periodic'', by increasing the latent dimension.
\begin{conjecture}
Suppose that $U\subset\R^n$ is a contractible open set, $f:U\to\R^n$ is a vector field, and $S$ is a smooth trivial $(2d-1)$-knot embedded in $U$. If $f|_{S}$ is a smooth tangent vector field on $S$, then there exists a smooth single-valued function $H(y)$ and a smooth matrix-valued function $B(y)$ satisfying Lemma \ref{lem:gen_brac} with rank of $2d$, such that $B(y)\nabla H(y)|_{S}=f|_{S}$.
\end{conjecture}
The conjecture provides a more general insight into the above theoretical results on single trajectory. If it holds, one may learn several trajectories lying on a higher-dimensional trivial knot simultaneously upon training, with higher latent dimension. Unfortunately, the proofs of 1-knot case cannot be easily extended to the general case, since the fact that a solenoidal vector field on $\R^2$ is exactly a field of Hamiltonian system does not hold for higher-dimensional space. A more thorough explanation of this conjecture is needed in future works.

\section{Simulation results}\label{sec:result}
In this section, we present several simulation cases to verify our theoretical results, and indicate the potential application of PNNs in the field of computer vision. All the hyper-parameters for detailed architectures and training settings are shown in Appendix \ref{sec:arch} and Table~\ref{tab:model_params}. For each detailed Poisson system involved, we obtain the ground truth and data using a high order symplectic integrator with its corresponding coordinate transformation, which is listed in Appendix \ref{sec:trans}.

\begin{figure*}[ht]
    \centering
    \includegraphics[width=0.9\textwidth]{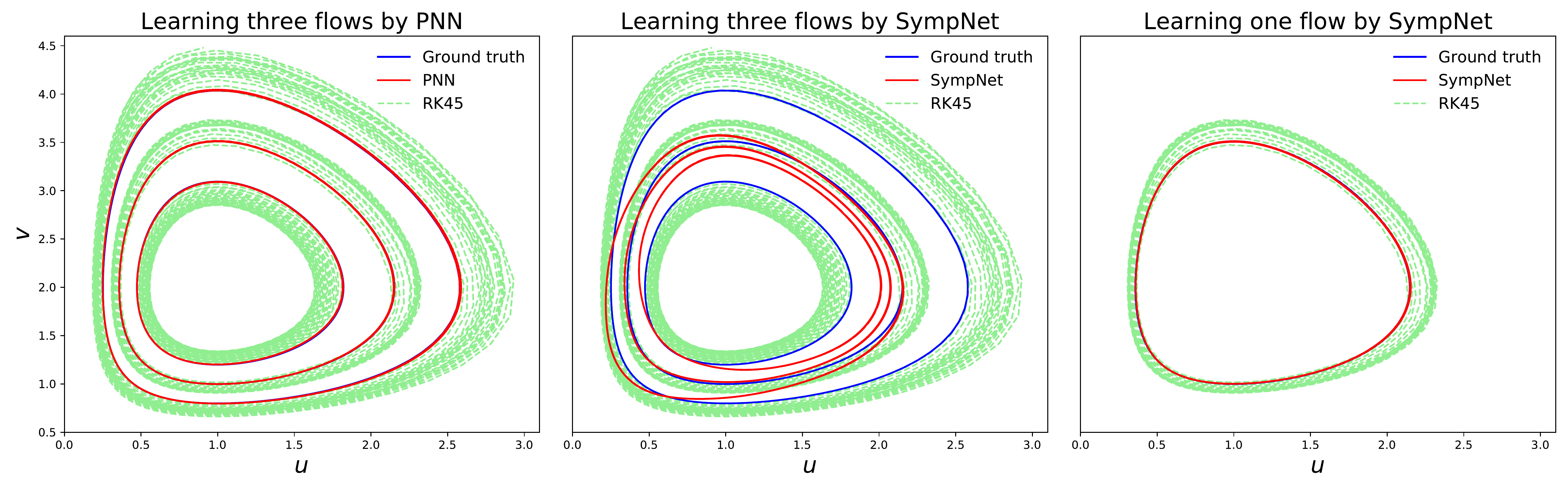}
    \caption{\textbf{Lotka–Volterra equation}. (\textbf{Left}) Learning three trajectories by PNN. The PNN successfully learns the system and achieves stable long time predictions, compared to the classical Runge-Kutta method of order four (RK45). (\textbf{Middle}) Learning three trajectories by a SympNet. The SympNet fails to fit the three trajectories simultaneously, since the learned system is not Hamiltonian. (\textbf{Right}) Learning a single trajectory by a SympNet. The SympNet is able to learn this single trajectory due to Theorem \ref{thm:2d_auto}.}
    \label{fig:LV}
\end{figure*}

\subsection{Lotka–Volterra equation} \label{ssec:lv}
The Lotka–Volterra equation can be written as
\begin{equation} \label{eq:LV}
    \begin{pmatrix} \dot{u} \\ \dot{v} \end{pmatrix}=\begin{pmatrix} u(v-2) \\ v(1-u) \end{pmatrix}=\begin{pmatrix} 0 & uv \\ -uv & 0 \end{pmatrix}\nabla H(u,v),
\end{equation}
where $H(u,v)=u-\ln u + v - 2 \ln v$. As a Poisson system, we are able to discover the underlying symplectic structure using PNNs. The data consist of three trajectories, starting at $(1,0.8),(1,1),(1,1.2)$, respectively. We generate 100 training points with time step $h=0.1$ for each trajectory. Besides the PNN, we also use a SympNet to learn the three trajectories simultaneously, as well as learn the single trajectory starting at $(1,1)$. We perform predictions for 1000 steps starting at the end points of the training trajectories, and the results of the three cases are presented in Fig.~\ref{fig:LV}. As shown in the left figure, the PNN successfully learns the system and achieves a stable long time prediction, compared to the classical Runge-Kutta method of order four (RK45). Meanwhile, the SympNet \cite{jin2020sympnets} fails to fit the three trajectories simultaneously, since the data points are not from a Hamiltonian system, as shown in the middle figure. However, the right figure reveals that the SympNet is indeed able to learn a single trajectory, even though it is not from a Hamiltonian system, which is consistent with Theorem \ref{thm:2d_auto}.

\subsection{Extended pendulum system} \label{ssec:pd}
We test the performance of a PNN on odd-dimensional Poisson systems. The motion of the pendulum system is governed by
\begin{equation*}
    \begin{pmatrix} \dot{p} \\ \dot{q} \end{pmatrix}=\begin{pmatrix} -\sin q \\ p \end{pmatrix}=\begin{pmatrix} 0 & -1 \\ 1 & 0 \end{pmatrix}\nabla H(p,q),
\end{equation*}
where $H(p,q)=\frac{1}{2}p^2 - \cos q$. This is a canonical Hamiltonian system, and we subsequently extend this system to three-dimensional space:
\begin{equation*}
    \begin{pmatrix} \dot{p} \\ \dot{q} \\ \dot{c} \end{pmatrix}=\begin{pmatrix} -\sin q \\ p+c \\ 0 \end{pmatrix}=\begin{pmatrix} 0 & -1 & 0 \\ 1 & 0 & 0 \\
    0 & 0 & 0
    \end{pmatrix}\nabla \Tilde{H}(p,q,c),
\end{equation*}
where $\Tilde{H}(p,q,c) = H(p,q)+pc$.
To make the data more difficult to learn, a nonlinear transformation
    $(u,v,r) = \theta^{-1}(p,q,c) = (p,q,p^2+q^2+c)$ is applied to the extended phase space.
The governing equation for the transformed system is as follows:
\begin{equation} \label{eq:PD}
\begin{split}
\begin{pmatrix}
\dot{u}\\ \dot{v} \\ \dot{r}
\end{pmatrix} &=
\begin{pmatrix}
0 & -1 & -2v\\ 1 & 0 & 2u \\ 2v & -2u & 0
\end{pmatrix}
\begin{pmatrix}
u-3u^2-v^2+r\\ \sin v -2uv \\ u
\end{pmatrix}\\ &=:B(u,v,r)\nabla K(u,v,r),
\end{split}
\end{equation}
where $K(u,v,r)=\frac{1}{2}u^2-\cos v +ur-u^3-uv^2$. One may readily verify that $B$ satisfies Lemma \ref{lem:gen_brac}, therefore (\ref{eq:PD}) is a Poisson system.

Three trajectories are simulated with initial conditions $x_0 = (0, 1, 1^2)$, $(0, 1.5, 1.5^2 + 0.1)$, $(0, 2, 2^2 + 0.2)$, and time step $h = 0.1$. We use data points obtained in the first 100 steps as our training set. Then we perform predictions for 1000 steps starting at the end points of training set; the results are shown in Fig.~\ref{fig:PD}. From the left figure, it can be seen that the predictions made by the PNN match the ground truth perfectly, remaining on the true trajectories after long times. Meanwhile, the PNN is able to recover the underlying structure of the system, as shown on the right hand side. The trajectories of the system in the latent space are recovered as trajectories on parallel planes, which matches the fact that the trajectories are generated from several different two-dimensional symplectic submanifolds.
\begin{figure*}[ht]
    \centering
    \includegraphics[width=0.8\textwidth]{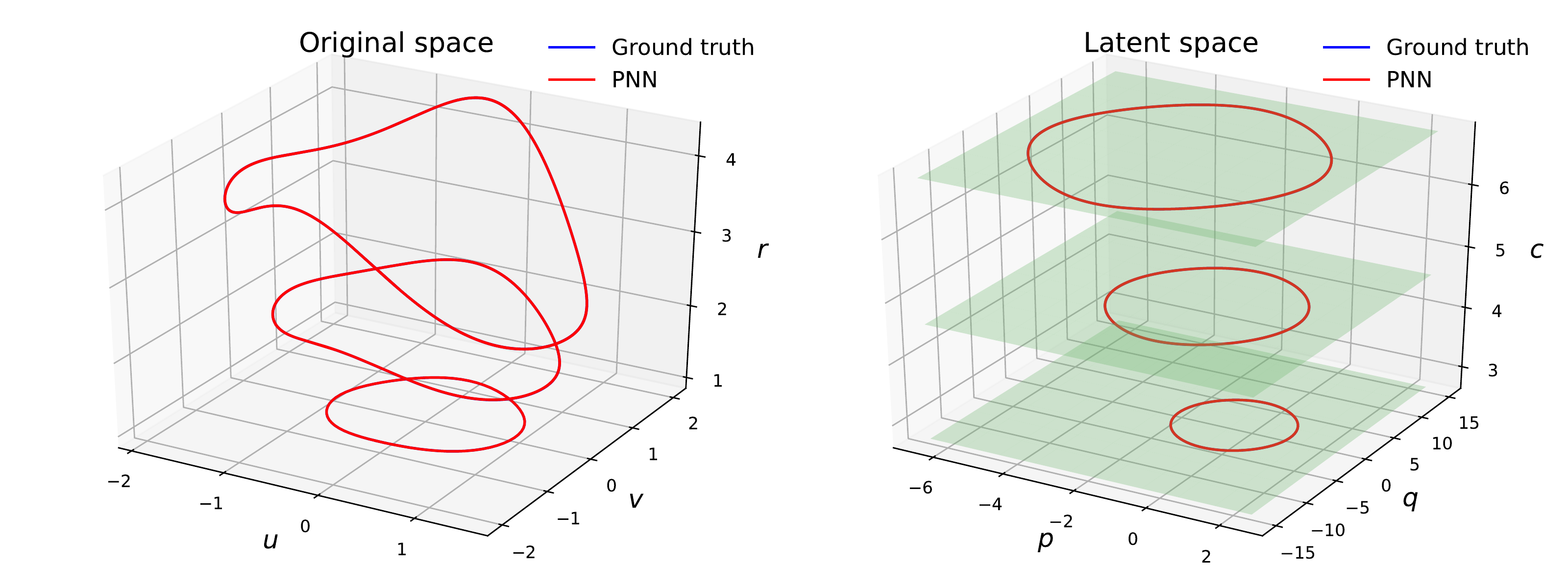}
    \caption{\textbf{Extended pendulum system.} (\textbf{Left}) Three ground truth trajectories compared to the predictions made by the PNN. The predictions match the ground truth perfectly without deviation. (\textbf{Right}) The ground truth and the predictions expressed in the latent space by trained $\theta$. The trajectories are lying on three different parallel planes.}
    \label{fig:PD}
\end{figure*}

\subsection{Charged particle in an electromagnetic potential} \label{ssec:lorentz}
We consider the dynamics of the charged particle in an electromagnetic potential governed by the Lorentz force
\begin{equation*}
    m\Ddot{x}=q(E+\dot{x}\times B),
\end{equation*}
where $m$ is the mass, $x\in\R^3$ denotes the particle’s position, $q$ is the electric charge, $B=\nabla\times A$ denotes the magnetic field, and $E=-\nabla\varphi$ is the electric field with $A,\varphi$ being the potentials. Let $\dot{x}=v$ be the velocity of the charged particle, then the governing equations of the particle's motion can be expressed as
\begin{equation} \label{eq:poi_lorentz}
\begin{split}
    &\begin{pmatrix}\dot{v} \\ \dot{x}\end{pmatrix}=\begin{pmatrix}-\frac{q}{m^2}\hat{B}(x) & -\frac{1}{m}I \\ \frac{1}{m}I & 0\end{pmatrix}\nabla H(v,x),\\ &H(v,x)=\frac{1}{2}mv^Tv+q\varphi(x),    
\end{split}
\end{equation}
where
\begin{equation*}
    \hat{B}(x)=\begin{pmatrix}0 & -B_3(x) & B_2(x) \\ B_3(x) & 0 & -B_1(x) \\ -B_2(x) & B_1(x) & 0\end{pmatrix}
\end{equation*}
for $B(x)=(B_1(x),B_2(x),B_3(x))$. Here we test the dynamics with $m=1$, $q=1$, and
\begin{equation*}
    A(x)=\frac{1}{3}\sqrt{x_1^2+x_2^2}\cdot(-x_2,x_1,0),\quad \varphi(x)=\frac{1}{100\sqrt{x_1^2+x_2^2}}
\end{equation*}
for $x=(x_1,x_2,x_3)^T$. Then
\begin{equation*}
\begin{split}
    &B(x)=(\nabla\times A)(x)=(0,0,\sqrt{x_1^2+x_2^2}),\\ &E(x)=-(\nabla\varphi)(x)=\frac{(x_1,x_2,0)}{100(x_1^2+x_2^2)^{\frac{3}{2}}}.
\end{split}
\end{equation*}

The initial state is chosen to be $v=(1,0.5,0)$, $x=(0.5,1,0)$, in which case the system degenerates into four-dimensional dynamics, i.e., the motion of the particle is always on a plane. For simplicity, we also denote them by $v=(1,0.5)$, $x=(0.5,1)$, and study the dimension-reduced system. We then generate a trajectory of 1500 training points followed by 300 test points with time step $h=0.1$. Subsequently, a volume-preserving PNN (VP-PNN) is trained to learn the training set, and we perform predictions for 2000 steps starting at the end point of the training set, as shown in Fig.~\ref{fig:LF}. It can be seen that the VP-PNN perfectly predicts the trajectory without deviation. Furthermore, we also train a non-volume-preserving PNN (NVP-PNN) and a volume-preserving neural network (VPNN) to compare with the above model. After sufficient training, the three models make predictions starting at the initial state to reconstruct trajectories, as shown in Fig.~\ref{fig:LF_comparison}. As one can see, the VP-PNN performs slightly better than the NVP-PNN, while the NVP-PNN is much better than the VPNN. The quantitative results shown in Table \ref{tab:LF_loss} also support this observation. Although the VP-PNN has larger training MSE and one step test MSE than the NVP-PNN, its long time test MSE is instead less, which is not surprising because NVP-PNNs and VPNNs possess the prior information of symplectic structure and volume preservation respectively, while VP-PNNs has both of them. Note that the considered dimension-reduced system of (\ref{eq:poi_lorentz}) is source-free hence its phase flow is intrinsically volume-preserving on the four-dimensional space.
\begin{figure}[htb]
    \centering
    \includegraphics[width=0.48\textwidth]{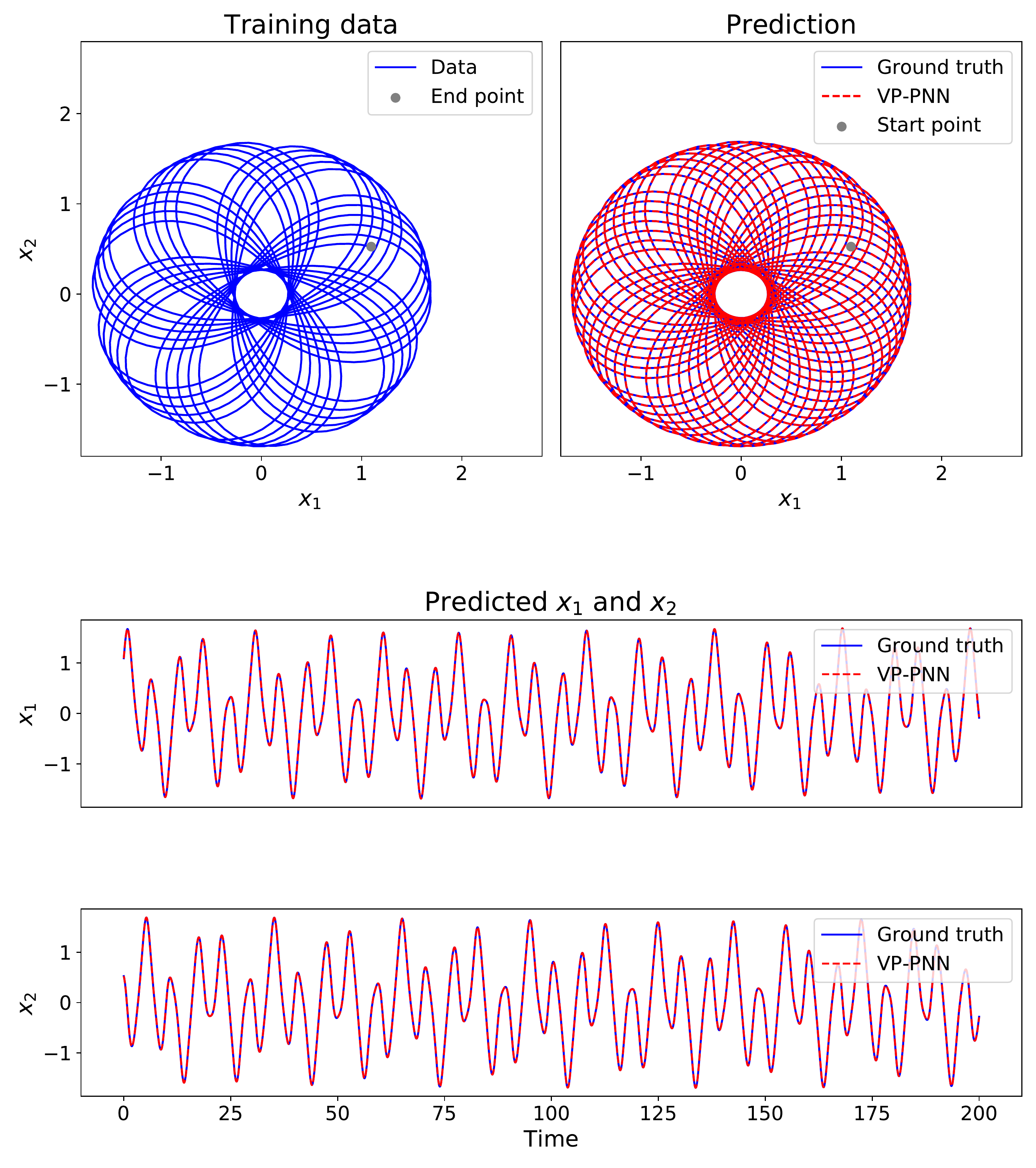}
    \caption{\textbf{Charged particle in the electromagnetic potential.} (\textbf{Top-left}) The position $(x_1,x_2)$ of training flow. (\textbf{Top-right}) Prediction of the VP-PNN starting at the end point of training flow for 2000 steps. The VP-PNN perfectly predicts the trajectory without deviation. (\textbf{Bottom}) Prediction of the position of particle over time.}
    \label{fig:LF}
\end{figure}

\begin{figure*}[htb]
    \centering
    \includegraphics[width=0.9\textwidth]{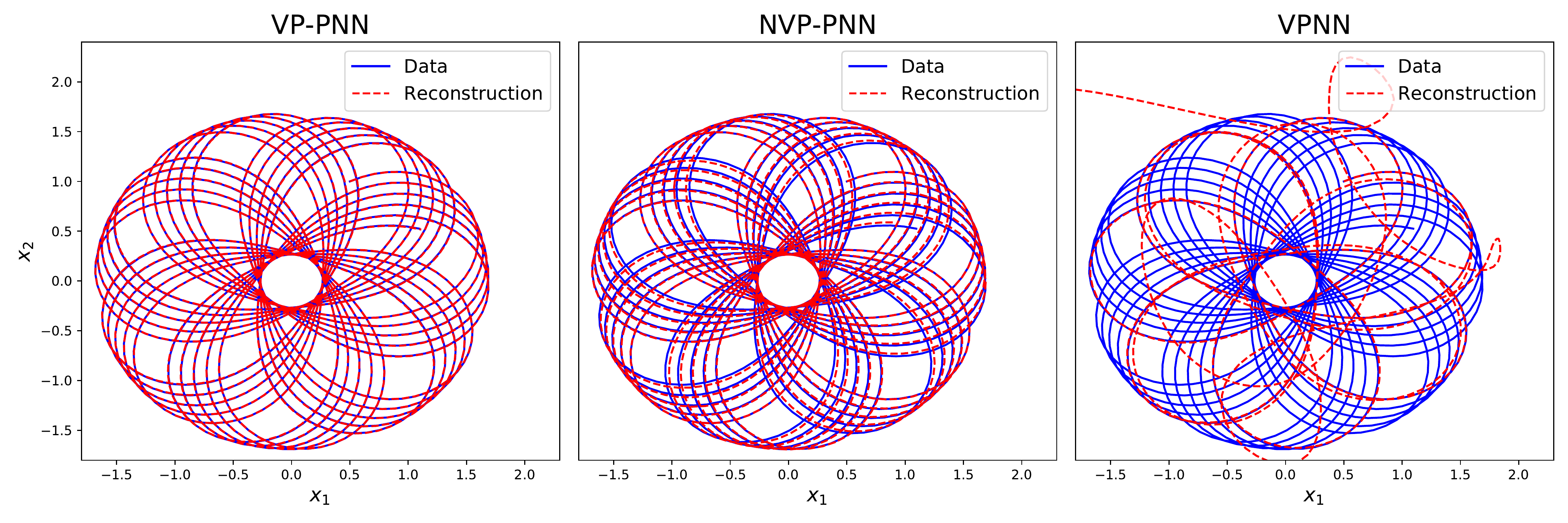}
    \caption{\textbf{Comparison of the reconstructed trajectories by three models.} (\textbf{Left}) The reconstructed trajectory by the VP-PNN, which perfectly matches the ground truth. (\textbf{Middle}) The reconstructed trajectory by the NVP-PNN, which performs well, but slightly worse than the VP-PNN. (\textbf{Right}) The reconstructed trajectory by the general VPNN, which fails to make long time prediction.}
    \label{fig:LF_comparison}
\end{figure*}

\begin{table}[htbp]
\centering
\begin{tabular}{|c|c|c|c|}
\hline
 & VP-PNN & NVP-PNN & VPNN \\ \hline
Training MSE  & $4.6\times 10^{-9}$ & $1.6\times 10^{-9}$ & $2.7\times 10^{-8}$ \\ \hline
 Test MSE (One step) & $1.5\times 10^{-8}$ & $9.8\times 10^{-9}$ & $6.1\times 10^{-8}$ \\
 \hline
  Test MSE (Long time) & $1.1\times 10^{-5}$ & $5.1\times 10^{-4}$ & $8.8\times 10^{3}$ \\
 \hline
\end{tabular}
\caption{\textbf{Losses of the three trained models.} The long time test MSE is evaluated along 1000 steps starting at the end point of training data. Although VP-PNN has larger training MSE and one step test MSE than NVP-PNN, its long time test MSE is instead less. The VPNN performs worse than above two models and fails for long time prediction.}
\label{tab:LF_loss}
\end{table}

\subsection{Nonlinear Schr{\"o}dinger equation} \label{ssec:al}
We consider the nonlinear Schr{\"o}dinger equation
\begin{equation*}
\left\{\begin{aligned}
&i\frac{\partial w}{\partial t}+\frac{\partial^2 w}{\partial x^2}+2|w|^2w=0, \\
&w(x,0)=w_0(x),
\end{aligned}\right.
\end{equation*}
where $w(x,t)=u(x,t)+iv(x,t)$ is a complex field and the boundary condition $w_0(x)$ is periodic, i.e., $w_0(x+1)=w_0(x)$ for $x\in\mathbb{R}$. An interesting space discretization of the nonlinear Schr{\"o}dinger equation is the Ablowitz–Ladik model
\begin{equation*}
    i\dot{w}_k+\frac{1}{\Delta x^2}(w_{k+1}-2w_k+w_{k-1})+|w_k^2|(w_{k+1}+w_{k-1})=0
\end{equation*}
with $w_k=w(k\Delta x, t)$, $\Delta x=1/N$. Letting  $w_k=u_k+iv_k$, we obtain
\begin{equation*}
    \begin{split}
        \dot{u_k}&=-\frac{1}{\Delta x^2}(v_{k+1}-2v_k+v_{k-1})-(u_k^2+v_k^2)(v_{k+1}+v_{k-1}),\\
        \dot{v_k}&=\frac{1}{\Delta x^2}(u_{k+1}-2u_k+u_{k-1})+(u_k^2+v_k^2)(u_{k+1}+u_{k-1}).
    \end{split}
\end{equation*}
With $u=(u_1,\cdots,u_N)$, $v=(v_1,\cdots,v_N)$ this system can be written as
\begin{equation} \label{eq:al_model}
    \begin{pmatrix}\dot{u} \\ \dot{v}\end{pmatrix}=\begin{pmatrix}0 & -D(u,v) \\ D(u,v) & 0\end{pmatrix}\nabla H(u,v)
\end{equation}
where $D=diag(d_1,\cdots,d_N)$ is the diagonal matrix with entries $d_k(u,v)=1+\Delta x^2(u_k^2+v_k^2)$, and the Hamiltonian is
\begin{equation*}
\begin{split}
    H(u,v)&=\frac{1}{\Delta x^2}\sum_{l=1}^{N}(u_lu_{l-1}+v_lv_{l-1})\\&-\frac{1}{\Delta x^4}\sum_{l=1}^{N}\ln(1+\Delta x^2(u_l^2+v_l^2)).
\end{split}
\end{equation*}
We thus get a Poisson system. In the experiment, we choose the boundary condition $u(x,0)=2+0.2\cdot\cos(2\pi x)$, $v(x,0)=0$ and set $N=20$, hence (\ref{eq:al_model}) is a Poisson system of dimension 40. We then generate 500 training points followed by 100 test points with time step $h=0.01$. That means the solution to this equation during the time interval $[0,5]$ is treated as the training set, and then we learn the data using a PNN and predict the solution between $[5,6]$. The result is shown in Fig.~\ref{fig:AL}: both of the real part and imaginary part match the ground truth well.

\begin{figure}[ht]
    \centering
    \includegraphics[width=0.48\textwidth]{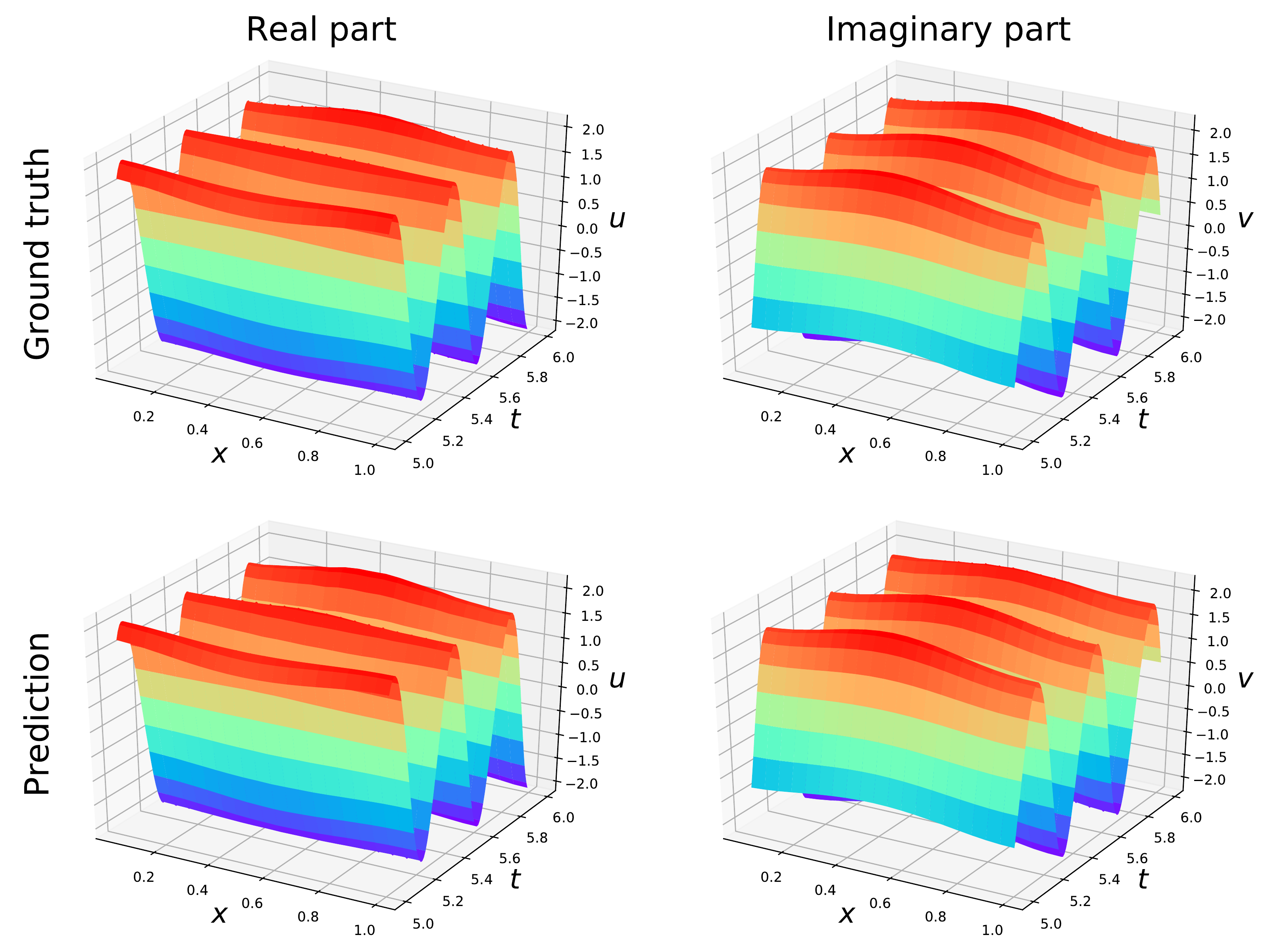}
    \caption{\textbf{Ablowitz–Ladik model of nonlinear Schr{\"o}dinger equation.} Predictions of both real and imaginary parts match the ground truth well.}
    \label{fig:AL}
\end{figure}

\begin{figure*}[ht]
    \centering
    \includegraphics[width=0.90\textwidth]{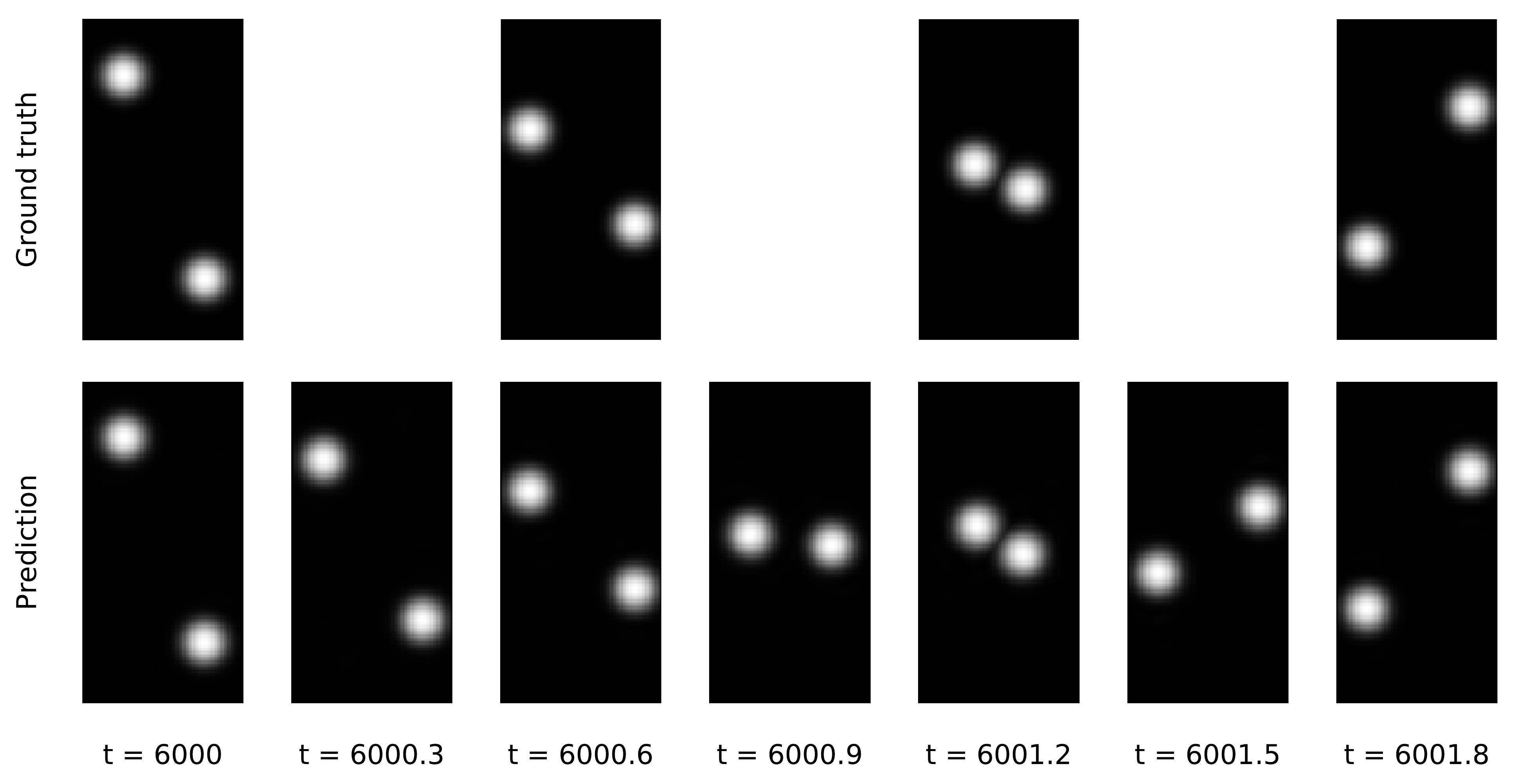}
    \caption{\textbf{Long time prediction and frame interpolation for two body images.} (\textbf{Top}) Ground truth, four consecutive points in the test dataset, after $t = 6000$ with step size $h = 0.6$. (\textbf{Bottom}) Predictions made by the PNN, with a finer step size $h=0.3$. (\textbf{All}) PNNs can handle long-time integration and frame interpolation perfectly.}
    \label{fig:TB_PNN}
\end{figure*}

\subsection{Pixel observations of two-body problem} \label{ssec:tb}
We consider the pixel observations of the two-body problem, which are the images of two balls in black background, as shown in Fig.~\ref{fig:TB_PNN}. Time series of the images form a movie of the motion of two balls governed by gravitation. Here we intend to learn the phase flow on a coarse time grid while making predictions on a finer time grid, to forecast and smoothen the movie simultaneously. To achieve this goal, a simple recurrent training scheme is applied to our method. Similar treatments can be found in \cite{anderson1996comparison}.

Suppose the training dataset is $\{x_n := x(n\Delta t)\}_{n=0}^{N}$. We set our goal to be making predictions on
$x(t)$ at $t = (N+\frac{n}{m})\Delta t$, $n\geq 1$. Denote the PNN to be trained as $f_{PNN} = \theta^{-1}\circ \Phi\circ \theta$. Then we train
\begin{equation*}
    f_{PNN}^m=\theta^{-1}\circ\Phi^m\circ \theta
\end{equation*}
to approximate the flow from $x_i$ to $x_{i+1}$, and $m$ is set to 2 in this case. Since we are learning a single trajectory on a submanifold of the high-dimensional space, autoencoders are used to approximate $\theta$. $\lambda$ in the corresponding loss function is chosen to be 1. After training, $f_{PNN}^k=\theta^{-1}\circ \Phi^k\circ \theta$ is used to generate predictions for $k\geq 1$.

A single trajectory $x(t)$ of the system is generated with time step $\Delta t = 0.6$, as shown in Fig.~\ref{fig:TB_PNN}. The training dataset contains $N = 100$ images of size $100\times 50$. Let $X_{grid} = (x_{N+1},\cdots,x_{N+k})$ and $\Tilde{X}_{grid} = (\Tilde{x}_{N+1},\cdots,\Tilde{x}_{N+k})$ denote the ground truth and predictions made by the PNN at grid points. Let $X_{mid} = (x_{N+\frac{1}{2}},\cdots,x_{N+k-\frac{1}{2}})$ and $\Tilde{X}_{mid} = (\Tilde{x}_{N+1},\cdots,\Tilde{x}_{N+k})$ denote the ground truth and predictions made by the PNN on middle points of the grids. The test loss on grids is calculated as the mean squared error between $X_{grid}$ and $\Tilde{X}_{grid}$ while the test loss on middle points is calculated as the mean squared error between $X_{mid}$ and $\Tilde{X}_{mid}$. We use a similar definition as in \cite{VLACHAS2020191} to compute the valid prediction time $T_\epsilon$. Suppose we are given the ground truth dataset $x$ and prediction $\tilde{x}$, starting from $t = 0$. Let the root mean square error (RMSE) be
\begin{equation*}
    \mathcal{E} (\tilde{x}) = \sqrt{\langle(x-\tilde{x})^2\rangle},
\end{equation*}
where $<\cdot>$ stands for spatial average. The valid prediction time (VPT) is defined to be
\begin{equation*}
    T_\epsilon = \argmax_{t_f} \{t_f|\mathcal{E}(\tilde{x}(t)) \leq \epsilon, \forall t\leq t_f\},
\end{equation*}
where $\epsilon$ is a hyperparameter to be chosen. Here we set $\epsilon = 0.02$.

Low error is obtained both on the grid points and the middle points, as shown in Table \ref{tab:TB_loss}, which indicates that PNNs can handle prediction and interpolation simultaneously. The VPT is much longer than the time scale of the training window, further suggesting that PNNs are good at long-time predictions and intrinsically structure-preserving. It can be seen in Fig.~\ref{fig:TB_PNN} that the prediction matches the ground truth perfectly even after $t = 6000$.

Note that according to Theorem \ref{thm:nd_auto}, the periodic solution to an autonomous dynamical system in $\R^d$ can always be learned by PNNs when $d>3$. Since this trajectory of two-body system is periodic in $\R^{100\times 50}$ and the image at step $n+1$ is uniquely determined by the image at step $n$, we can assume without loss of generality that the pixel observations of a two-body system form a periodic solution to an autonomous dynamical system, regardless of whether the internal mechanism is Hamiltonian. 

\begin{table}[htbp]
\centering
\begin{tabular}{|c|c|c|c|}
\hline
 Train MSE & Test MSE (Grid) & Test MSE (Middle) & VPT \\ \hline
 $1.7\times 10^{-6}$ & $2.1\times 10^{-6}$ & $2.0\times 10^{-6}$ & 6308 \\ \hline
\end{tabular}
\caption{\textbf{Losses and VPT of the PNN.} The test MSE is evaluated along 100 steps starting at the end point of data.}
\label{tab:TB_loss}
\end{table}

\section{Summary}\label{sec:summary}
The main contribution of this paper is to provide a novel high-level network architecture, PNN, to learn the phase flow of an arbitrary Poisson system. Since a single periodic solution to an autonomous system can be proven to be a solution to a Poisson system if the orbit is unknotted, PNNs can be directly applied to a much broader class of systems without modification. From this perspective, theoretical results regarding the approximation ability of PNNs are presented. Several simulations including the Lotka-Volterra equation, an extended pendulum system, charged particles in the electromagnetic potential, a nonlinear Schr{\"o}dinger equation and a trajectory of the two-body problem support our theoretical findings and illustrate the advantages of PNNs on long time prediction and frame interpolation. 
Even though not explicitly mentioned in the paper, PNNs can be easily extended to learn phase flows from irregularly sampled data. Interested readers may refer to \cite{jin2020sympnets} for more details. PNNs can also learn the Hamiltonian systems on low dimensional submanifolds or constrained Hamiltonian systems, which can be expressed as a Poisson system on local coordinates \cite[Chapter VII.1]{hairer2006geometric}.

Despite the great expressivity, stability and interpretability of PNNs, an open issue is whether one can use it to infer and make predictions on multiple trajectories of an autonomous system without generalized Poisson structure under certain circumstances. We conjectured in the paper that the solution to an arbitrary autonomous system lying on a smooth trivial $(2d-1)$-knot matches the solution to a Poisson system. If this holds, the use of PNNs on multiple trajectories of a general autonomous system would be theoretically justified. We leave the proof or counterexamples of this conjecture as future work.

\appendices
\section{Introduction to knot} \label{sec:knot}
Consider the embeddings of $S^n$ in $S^m$, $m\geq n+1$. Two embeddings $k_1,k_2$ are equivalent if there is a homeomorphism $h$ of $S^m$ such that $h(k_1)=k_2$. An embedding $k:S^n\subset S^m$ is \emph{unknotted} if it is equivalent to the trivial knot $k_0:S^n\subset S^m$ defined by the standard embedding
\begin{equation*}
    k_0:S^n\to S^m;\ (x_0,\cdots,x_n)\to(x_0,\cdots,x_n,0,\cdots,0).
\end{equation*}
In fact, the embedding $k$ is always unknotted when $m\neq n+2$. Therefore an $n$-knot is defined as an embedding of $S^n$ in $S^{n+2}$. For convenience, let us make a little change in the terminology: an $n$-knot means an embedding of $S^{n}$ in $S^{m}$ for $m\geq n+1$, since we are more concerned about the trivial knot in this work.

\section{Proofs for theorems}
\subsection{Proof of Theorem \ref{thm:2d_auto}} \label{proof:2d_auto}
In two-dimensional space, consider the system
\begin{equation*}
\left\{\begin{aligned}
&\Delta \boldsymbol{v}={\rm grad}\ p, \\
&{\rm div}\ \boldsymbol{v} = 0,
\end{aligned}\right.
\end{equation*}
with boundary condition
\begin{equation*}
    \boldsymbol{v}|_\Gamma=\boldsymbol{a},
\end{equation*}
where $\Gamma$ is the orbit of $y(t)$ and $\boldsymbol{v}:U\to\R^2$ is a vector field. \cite[p. 60, Theorem 1]{ladijzenskaia1969mathematical} shows that there exists a solution $\boldsymbol{v}$ to the system above given a suitable single-valued function $p$ and any continuous $\boldsymbol{a}$ satisfying
\begin{equation*}
    \oint_\Gamma \boldsymbol{a}\cdot\boldsymbol{n}ds=0,
\end{equation*}
where $\boldsymbol{n}$ is the exterior normal with respect to the domain inside $\Gamma$. Set $\boldsymbol{a}=f$, we have
\begin{equation*}
    \oint_\Gamma f\cdot\boldsymbol{n}ds=\oint_\Gamma 0ds=0,
\end{equation*}
hence there exists a $\boldsymbol{v}$ such that
\begin{equation*}
    {\rm div}\ \boldsymbol{v}=0,\quad \boldsymbol{v}|_\Gamma=f.
\end{equation*}
Note that $\boldsymbol{v}$ is a solenoidal vector field, which leads to
\begin{equation*}
    \boldsymbol{v} = {\rm curl}\ (-H)=(-\frac{\partial H}{\partial y_2}, \frac{\partial H}{\partial y_1})^T=J^{-1}\nabla H.
\end{equation*}
The $H$ defined above is exactly the Hamiltonian we are looking for.
\subsection{Proof of Theorem \ref{thm:nd_auto}} \label{proof:nd_auto}
Since $y(t)$ is unknotted, there exist a periodic $\gamma(t)=(\gamma_1(t),\gamma_2(t))^T\in\R^2$ and a homeomorphism $\theta:U\to U$ such that
\begin{equation*}
    \theta(y(t))=\Gamma(t):=(\gamma_1(t),\gamma_2(t),0,\cdots,0)^T\in\R^n,
\end{equation*}
which is exactly the coordinate transformation deforming $y(t)$ into $\Gamma(t)$. Theorem \ref{thm:2d_auto} shows that $\gamma(t)$ is also a solution to a Hamiltonian system, hence the extended solution $\Gamma(t)$ satisfies
\begin{equation} \label{eq:gamma}
\left\{\begin{aligned}
&\dot{\Gamma}=\begin{pmatrix}
J^{-1} & 0 \\ 0 & 0
\end{pmatrix}\nabla K(\Gamma),\quad J\in\R^{2\times 2}, \\
&K(y_1,y_2,\cdots,y_n)=H(y_1,y_2),
\end{aligned}\right.
\end{equation}
for a single-valued function $H$. In fact, a Poisson system expressed in an arbitrary new coordinate immediately becomes a new Poisson system with a new $B(y)$ whose rank is equivalent to the original one \cite[p. 265]{hairer2006geometric}. Therefore, $y(t)$ is also a solution to a Poisson system obtained by expressing system (\ref{eq:gamma}) in new coordinate via transformation $\theta$. Furthermore, they share the same latent dimension of 2.

\section{Implementation of architecture}\label{sec:arch}

\begin{table*}[htbp]
\centering
\begin{tabular}{|c|c|c|c|c|c|c|c|c|}
\hline
\multicolumn{2}{|c|}{\multirow{2}{*}{Problem}} & \multicolumn{2}{c|}{Section \ref{ssec:lv}} & \multirow{2}{*}{Section \ref{ssec:pd}} & \multicolumn{2}{c|}{Section \ref{ssec:lorentz}} & \multirow{2}{*}{Section \ref{ssec:al}} & \multirow{2}{*}{Section \ref{ssec:tb}} \\ \cline{3-4} \cline{6-7}
\multicolumn{2}{|c|}{} & PNN & SympNet &  & PNN & VPNN &  &  \\ \hline
\multirow{5}{*}{$\theta$} & Type & NVP & - & NVP & VP/NVP & VP & NVP & AE \\ \cline{2-9}
 & Partition & 1 & - & 2 & 2 & 2 & 20 & - \\ \cline{2-9}
 & Layers & 3 & - & 3 & 10 & 20 & 3 & 2 \\ \cline{2-9}
 & Sublayers & 2 & - & 2 & 3 & 3 & 2 & - \\ \cline{2-9}
 & Width & 30 & - & 30 & 50 & 50 & 100 & 50 \\ \hline
\multirow{4}{*}{$\Phi$} & Type & G & G & E & G & - & G & LA \\ \cline{2-9}
 & Layers & 3 & 6 & 3 & 10 & - & 10 & 3 \\ \cline{2-9}
 & Sublayers & - & - & - & - & - & - & 2 \\ \cline{2-9}
 & Width & 30 & 30 & 30 & 50 & - & 100 & - \\ \hline
\end{tabular}
\caption{\textbf{Model architecture.} The detailed implementations of symplectic neural networks, invertible neural networks as well as the autoencoder, and their corresponding terminology used in this table are shown in Appendix \ref{sec:arch}. The Sigmoid activation functions are used for all the models. The optimizer is chosen to be Adam \cite{adam2015} with learning rate 0.001. The numbers of training iterations are $2\times 10^{5}$, $1\times 10^{5}$, $2\times 10^{6}$, $1\times 10^{6}$, and $5\times 10^{5}$ respectively for the five problems considered.}
\label{tab:model_params}
\end{table*}

\subsection{Extended symplectic neural networks}
\label{impl:sympnet}
We adopt SympNets \cite{jin2020sympnets} as the universal approximators for symplectic maps. The architecture of SympNets is based on three modules, i.e.,
\begin{itemize}
    \item \textbf{Linear modules.} \begin{equation*}
    \begin{split}
    &\mathcal{L}_{n}\begin{pmatrix} p \\ q \end{pmatrix}\\
    =&\begin{pmatrix} I & 0/S_{n} \\ S_{n}/0 & I \end{pmatrix}\cdots\begin{pmatrix} I & 0 \\ S_{2} & I \end{pmatrix}\begin{pmatrix} I & S_{1} \\ 0 & I \end{pmatrix}\begin{pmatrix} p \\ q \end{pmatrix}+b,\\
    &p,q\in\R^d,    
    \end{split}
    \end{equation*}
    where $S_i\in\R^{d\times d}$ are symmetric, $b\in\R^{2d}$ is the bias, while the unit upper triangular symplectic matrices and the unit lower triangular symplectic matrices appear alternately. In this module, $S_i$ (represented by $A_i+A_i^T$ in practice) and $b$ are parameters to learn. In fact, $\mathcal{L}_{n}$ can represent any linear symplectic map \cite{jin2020unit}.
    \item \textbf{Activation modules.} \begin{equation*}
    \begin{split}
    &\mathcal{N}_{up}\begin{pmatrix} p \\ q \end{pmatrix}=\begin{pmatrix} p+\Tilde{\sigma}_a(q) \\ q \end{pmatrix},\\ &\mathcal{N}_{low}\begin{pmatrix} p \\ q \end{pmatrix}=\begin{pmatrix} p \\ \Tilde{\sigma}_a(p)+q \end{pmatrix},\quad p,q\in\R^d,    
    \end{split}
    \end{equation*}
    where $\Tilde{\sigma}_a(x):=a\odot\sigma(x)$ for $x\in\R^d$. Here $\odot$ is the element-wise product, $\sigma$ is the activation function, and $a\in\R^d$ is the parameter to learn.
    \item \textbf{Gradient modules.} \begin{equation*}
    \begin{split}
    &\mathcal{G}_{up}\begin{pmatrix} p \\ q \end{pmatrix}=\begin{pmatrix} p+\hat{\sigma}_{K,a,b}(q) \\ q \end{pmatrix},\\ &\mathcal{G}_{low}\begin{pmatrix} p \\ q \end{pmatrix}=\begin{pmatrix} p \\ \hat{\sigma}_{K,a,b}(p)+q \end{pmatrix},\quad p,q\in\R^d,    
    \end{split}
    \end{equation*}
    where $\hat{\sigma}_{K,a,b}(x):=K^T(a\odot\sigma(Kx+b))$ for $x\in\R^d$. Here $a,b\in\R^l$, $K\in\R^{l\times d}$ are the parameters to learn, and $l$ is a positive integer regarded as the width of the module.
\end{itemize}
The SympNets are the composition of above three modules. In particular, we use two classes of SympNets: the LA-SympNets composed of linear and activation modules, and the G-SympNets composed of gradient modules. Notice that both LA and G SympNets are universal approximators for symplectic maps as shown in \cite{jin2020sympnets}. For convenience, we clarify the terminology for describing a detailed LA(G)-SympNet: an LA-SympNet of $m$ layers with $n$ sublayers means it is the composition of $m$ linear modules and $m-1$ activation modules, where linear and activation modules appear alternatively like the architecture of fully-connected neural network, and each linear module is composed of $n$ alternated triangular symplectic matrices; a G-SympNet of $m$ layers with width $l$ means it is composed of $m$ alternated gradient modules, and the width $l$ is defined as above.

In this works we further develop the extended symplectic neural networks by extending the gradient modules.
\begin{itemize}
    \item \textbf{Extended modules.} \begin{equation*}
    \begin{split}
    &\mathcal{E}_{up}\begin{pmatrix} p \\ q \\ c \end{pmatrix}=\begin{pmatrix} p+\widehat{\sigma}_{K_1,K_2,a,b}(q,c) \\ q \\ c \end{pmatrix},\\ &\mathcal{E}_{low}\begin{pmatrix} p \\ q \\ c \end{pmatrix}=\begin{pmatrix} p \\ \widehat{\sigma}_{K_1,K_2,a,b}(p.c)+q \\ c \end{pmatrix},\\
    &p,q\in\R^d,\ c\in\R^{n-2d},
    \end{split}
    \end{equation*}
    where $\widehat{\sigma}_{K_1,K_2,a,b}(x,c):=K_1^T(a\odot\sigma(K_1x+K_2c+b))$ for $x\in\R^d$, $c\in\R^{n-2d}$. Here $a,b\in\R^l$, $K_1\in\R^{l\times d}$, $K_2\in\R^{l\times (n-2d)}$ are the parameters to learn, and $l$ is a positive integer regarded as the width of the module.
\end{itemize}
The extended symplectic neural networks (E-SympNets) are the composition of extended modules. As noticed, $2d$ is the latent dimension of the E-SympNets. The terminology for describing the architecture of E-SympNets is the same as that for G-SympNets, hence we will not repeat here. It is worth mentioning that the approximation property of E-SympNets is still unknown, which is left as future work. 

\subsection{Invertible neural networks}
In numerical experiments, we adopt the NICE \cite{dinh2014nice} as volume-preserving invertible neural networks, and the real NVP \cite{dinh2016density} as general non-volume-preserving invertible neural networks.
\begin{itemize}
    \item \textbf{Volume-preserving modules.}
    \begin{equation*}
    \begin{split}
        & \mathcal{V}_{up}\begin{pmatrix}
        x_1 \\ x_2
        \end{pmatrix}=\begin{pmatrix}
        x_1+m_1(x_2) \\ x_2
        \end{pmatrix},\\ &\mathcal{V}_{low}\begin{pmatrix}
        x_1 \\ x_2
        \end{pmatrix}=\begin{pmatrix}
        x_1 \\ m_2(x_1)+x_2
        \end{pmatrix},\\ 
        &x_1\in\R^d,x_2\in\R^{n-d},
    \end{split}
    \end{equation*}
    where $m_1:\R^{n-d}\to\R^d$ and $m_2:\R^d\to\R^{n-d}$ are modeled as fully-connected neural networks.
    \item \textbf{Non-volume-preserving modules.}
    \begin{equation*}
    \begin{split}
        &\mathcal{C}_{up}\begin{pmatrix}
        x_1 \\ x_2
        \end{pmatrix}=\begin{pmatrix}
        x_1\odot\exp{s_1(x_2)}+t_1(x_2) \\ x_2
        \end{pmatrix}, \\ &\mathcal{C}_{low}\begin{pmatrix}
        x_1 \\ x_2
        \end{pmatrix}=\begin{pmatrix}
        x_1 \\ x_2\odot\exp{s_2(x_1)}+t_2(x_1)
        \end{pmatrix},\\ &x_1\in\R^d,x_2\in\R^{n-d},
    \end{split}
    \end{equation*}
    where $s_1,t_1:\R^{n-d}\to\R^d$ and $s_2,t_2:\R^d\to\R^{n-d}$ are modeled as fully-connected neural networks.
\end{itemize}
Basically, the volume-preserving (non-volume-preserving) invertible neural networks are the alternated composition of upper and lower volume-preserving (non-volume-preserving) modules. We say a volume-preserving (non-volume-preserving) INN is of $m$ layers and $n$ sublayers with width $l$ if it is composed of $m$ alternated volume-preserving (non-volume-preserving) modules and each of $m_1,m_2$ ($s_1,t_1,s_2,t_2$) is a FNN of $n$ layers with width $l$. Meanwhile, the partition dimension $d$ is also fixed as a architecture parameter.

\subsection{Autoencoder}
We apply the traditional autoencoder \cite{kramer1991nonlinear} to the alternative architecture for dimension-reduced cases. Suppose that the latent dimension is ${\rm rank}(B(y))=2d<n$, then the aotuencoder is composed of two fully-connected neural networks, i.e., an encoder $f_e:\R^n\to\R^{2d}$ and a decoder $f_d:\R^{2d}\to\R^n$, whose hidden nodes are all not less than $2d$. Note that the FNNs can also be replaced by other useful architectures like CNNs, if needed. In Section \ref{sec:result}, the encoder and decoder are chosen of same depth and width.

\section{Transformations for Poisson systems} \label{sec:trans}
Here we briefly present the transformations to Hamiltonian systems for the involved Poisson systems in this work. \\
\textbf{Lotka–Volterra equation.} With coordinate transformation $(p,q)=(\ln u,\ln v)$, system (\ref{eq:LV}) can be written as
\begin{equation*}
    \begin{pmatrix} \dot{p} \\ \dot{q} \end{pmatrix}=J^{-1}\nabla H(p,q),\quad H(p,q)=p-\exp(p)+2q-\exp(q).
\end{equation*}
\textbf{Charged particle in the electromagnetic potential.} By considering the position $x$ and the conjugate momentum $p=m\dot{x}+qA(x)$, the Poisson system (\ref{eq:poi_lorentz}) can be written in the canonical form
\begin{equation*}
\begin{split}
    &\begin{pmatrix}\dot{p} \\ \dot{x}\end{pmatrix}=J^{-1}\nabla H(p,x), \\ &H(p,x)=\frac{1}{2m}(p-qA(x))^T(p-qA(x))+q\varphi(x).
\end{split}
\end{equation*} \\
\textbf{Nonlinear Schr{\"o}dinger equation.} A transformation has been proposed to make (\ref{eq:al_model}) a canonical Hamiltonian system:
\begin{equation*}
\left\{\begin{aligned}
&p_k=u_k\sigma(\Delta x^2(u_k^2+v_k^2)) \\
&q_k=v_k\sigma(\Delta x^2(u_k^2+v_k^2))
\end{aligned}\right.,\quad {\rm with} \quad \sigma(x)=\sqrt{\frac{\ln(1+x)}{x}},
\end{equation*}
which treats the variables symmetrically. Its inverse is
\begin{equation*}
\left\{\begin{aligned}
&u_k=p_k\tau(\Delta x^2(p_k^2+q_k^2)) \\
&v_k=q_k\tau(\Delta x^2(p_k^2+q_k^2))
\end{aligned}\right.,\quad {\rm with} \quad \tau(x)=\frac{\exp{x}-1}{x}.
\end{equation*}
Now the Hamiltonian in the new variables is
\begin{equation*}
\begin{split}
    H(p,q)=&\frac{1}{\Delta x^2}\sum_{l=1}^{N}\tau(\Delta x^2(p_l^2+q_l^2))\tau(\Delta x^2(p_{l-1}^2+q_{l-1}^2))\\&\cdot(p_lp_{l-1}+q_lq_{l-1})-\frac{1}{\Delta x^2}\sum_{l=1}^{N}(p_l^2+q_l^2).
\end{split}
\end{equation*}

\ifCLASSOPTIONcaptionsoff
  \newpage
\fi



\bibliographystyle{IEEEtran}
\bibliography{IEEEabrv,reference}

\begin{thebibliography}{10}
\providecommand{\url}[1]{#1}
\csname url@samestyle\endcsname
\providecommand{\newblock}{\relax}
\providecommand{\bibinfo}[2]{#2}
\providecommand{\BIBentrySTDinterwordspacing}{\spaceskip=0pt\relax}
\providecommand{\BIBentryALTinterwordstretchfactor}{4}
\providecommand{\BIBentryALTinterwordspacing}{\spaceskip=\fontdimen2\font plus
\BIBentryALTinterwordstretchfactor\fontdimen3\font minus
  \fontdimen4\font\relax}
\providecommand{\BIBforeignlanguage}[2]{{%
\expandafter\ifx\csname l@#1\endcsname\relax
\typeout{** WARNING: IEEEtran.bst: No hyphenation pattern has been}%
\typeout{** loaded for the language `#1'. Using the pattern for}%
\typeout{** the default language instead.}%
\else
\language=\csname l@#1\endcsname
\fi
#2}}
\providecommand{\BIBdecl}{\relax}
\BIBdecl

\bibitem{chang2017reversible}
B.~Chang, L.~Meng, E.~Haber, L.~Ruthotto, D.~Begert, and E.~Holtham,
  ``Reversible architectures for arbitrarily deep residual neural networks,''
  \emph{arXiv preprint arXiv:1709.03698}, 2017.

\bibitem{weinan2017proposal}
W.~E, ``A proposal on machine learning via dynamical systems,''
  \emph{Communications in Mathematics and Statistics}, vol.~5, no.~1, pp.
  1--11, 2017.

\bibitem{haber2017stable}
E.~Haber and L.~Ruthotto, ``Stable architectures for deep neural networks,''
  \emph{Inverse Problems}, vol.~34, no.~1, p. 014004, 2017.

\bibitem{lu2018beyond}
Y.~Lu, A.~Zhong, Q.~Li, and B.~Dong, ``Beyond finite layer neural networks:
  Bridging deep architectures and numerical differential equations,'' in
  \emph{International Conference on Machine Learning}.\hskip 1em plus 0.5em
  minus 0.4em\relax PMLR, 2018, pp. 3276--3285.

\bibitem{lu2020mean}
Y.~Lu, C.~Ma, Y.~Lu, J.~Lu, and L.~Ying, ``A mean-field analysis of deep resnet
  and beyond: Towards provable optimization via overparameterization from
  depth,'' \emph{arXiv preprint arXiv:2003.05508}, 2020.

\bibitem{chen2018neural}
R.~T. Chen, Y.~Rubanova, J.~Bettencourt, and D.~K. Duvenaud, ``Neural ordinary
  differential equations,'' in \emph{Advances in neural information processing
  systems}, 2018, pp. 6571--6583.

\bibitem{rico1994continuous}
R.~Rico-Martinez, J.~Anderson, and I.~Kevrekidis, ``Continuous-time nonlinear
  signal processing: a neural network based approach for gray box
  identification,'' in \emph{Proceedings of IEEE Workshop on Neural Networks
  for Signal Processing}.\hskip 1em plus 0.5em minus 0.4em\relax IEEE, 1994,
  pp. 596--605.

\bibitem{yannis1998rk}
R.~González-García, R.~Rico-Martínez, and I.~Kevrekidis, ``Identification of
  distributed parameter systems: A neural net based approach,'' \emph{Computers
  \& Chemical Engineering}, vol.~22, pp. S965 -- S968, 1998, european Symposium
  on Computer Aided Process Engineering-8.

\bibitem{raissi2018multistep}
M.~Raissi, P.~Perdikaris, and G.~E. Karniadakis, ``Multistep neural networks
  for data-driven discovery of nonlinear dynamical systems,'' \emph{arXiv
  preprint arXiv:1801.01236}, 2018.

\bibitem{zhu2020inverse}
A.~Zhu, P.~Jin, and Y.~Tang, ``Inverse modified differential equations for
  discovery of dynamics,'' \emph{arXiv preprint arXiv:2009.01058}, 2020.

\bibitem{feng1984difference}
K.~Feng, ``On difference schemes and symplectic geometry,'' in
  \emph{Proceedings of the 5th international symposium on differential geometry
  and differential equations}, 1984.

\bibitem{hairer2006geometric}
E.~Hairer, C.~Lubich, and G.~Wanner, \emph{Geometric numerical integration:
  structure-preserving algorithms for ordinary differential equations}.\hskip
  1em plus 0.5em minus 0.4em\relax Springer Science \& Business Media, 2006,
  vol.~31.

\bibitem{lubich2008quantum}
C.~Lubich, \emph{From quantum to classical molecular dynamics: reduced models
  and numerical analysis}.\hskip 1em plus 0.5em minus 0.4em\relax European
  Mathematical Society, 2008.

\bibitem{tom2019ham}
T.~Bertalan, F.~Dietrich, I.~Mezić, and I.~G. Kevrekidis, ``On learning
  {Hamiltonian} systems from data,'' \emph{Chaos: An Interdisciplinary Journal
  of Nonlinear Science}, vol.~29, no.~12, p. 121107, 2019.

\bibitem{greydanus2019hamiltonian}
S.~Greydanus, M.~Dzamba, and J.~Yosinski, ``{Hamiltonian} neural networks,'' in
  \emph{Advances in Neural Information Processing Systems}, 2019, pp.
  15\,353--15\,363.

\bibitem{rezende2019equivariant}
D.~J. Rezende, S.~Racani{\`e}re, I.~Higgins, and P.~Toth, ``Equivariant
  {Hamiltonian} flows,'' \emph{arXiv preprint arXiv:1909.13739}, 2019.

\bibitem{sanchez2019hamiltonian}
A.~Sanchez-Gonzalez, V.~Bapst, K.~Cranmer, and P.~Battaglia, ``{Hamiltonian}
  graph networks with {ODE} integrators,'' \emph{arXiv preprint
  arXiv:1909.12790}, 2019.

\bibitem{chen2020symplectic}
Z.~Chen, J.~Zhang, M.~Arjovsky, and L.~Bottou, ``Symplectic recurrent neural
  networks,'' in \emph{8th International Conference on Learning
  Representations, {ICLR} 2020, Addis Ababa, Ethiopia, April 26-30,
  2020}.\hskip 1em plus 0.5em minus 0.4em\relax OpenReview.net, 2020.

\bibitem{Toth2020Hamiltonian}
P.~Toth, D.~J. Rezende, A.~Jaegle, S.~Racanière, A.~Botev, and I.~Higgins,
  ``{Hamiltonian} generative networks,'' in \emph{International Conference on
  Learning Representations}, 2020.

\bibitem{Zhong2020Symplectic}
Y.~D. Zhong, B.~Dey, and A.~Chakraborty, ``Symplectic {ODE-Net}: Learning
  {Hamiltonian} dynamics with control,'' in \emph{International Conference on
  Learning Representations}, 2020.

\bibitem{jin2020sympnets}
P.~Jin, Z.~Zhang, A.~Zhu, Y.~Tang, and G.~E. Karniadakis, ``{SympNets}:
  Intrinsic structure-preserving symplectic networks for identifying
  {Hamiltonian} systems,'' \emph{Neural Networks}, vol. 132, pp. 166--179,
  2020.

\bibitem{dipietro2020sparse}
D.~DiPietro, S.~Xiong, and B.~Zhu, ``Sparse symplectically integrated neural
  networks,'' \emph{Advances in Neural Information Processing Systems},
  vol.~33, 2020.

\bibitem{xiong2020nonseparable}
S.~Xiong, Y.~Tong, X.~He, C.~Yang, S.~Yang, and B.~Zhu, ``Nonseparable
  symplectic neural networks,'' \emph{arXiv preprint arXiv:2010.12636}, 2020.

\bibitem{cranmer2020lagrangian}
M.~Cranmer, S.~Greydanus, S.~Hoyer, P.~Battaglia, D.~Spergel, and S.~Ho,
  ``Lagrangian neural networks,'' \emph{arXiv preprint arXiv:2003.04630}, 2020.

\bibitem{finzi2020simplifying}
M.~Finzi, K.~A. Wang, and A.~G. Wilson, ``Simplifying {Hamiltonian} and
  {Lagrangian} neural networks via explicit constraints,'' \emph{Advances in
  Neural Information Processing Systems}, vol.~33, 2020.

\bibitem{tang1996symplectic}
Y.~Tang, L.~V{\'a}zquez, F.~Zhang, and V.~P{\'e}rez-Garc{\'\i}a, ``Symplectic
  methods for the nonlinear {S}chr{\"o}dinger equation,'' \emph{Computers \&
  Mathematics with Applications}, vol.~32, no.~5, pp. 73--83, 1996.

\bibitem{dinh2014nice}
L.~Dinh, D.~Krueger, and Y.~Bengio, ``Nice: Non-linear independent components
  estimation,'' \emph{arXiv preprint arXiv:1410.8516}, 2014.

\bibitem{dinh2016density}
L.~Dinh, J.~Sohl-Dickstein, and S.~Bengio, ``Density estimation using real
  {NVP},'' \emph{arXiv preprint arXiv:1605.08803}, 2016.

\bibitem{rico1992discrete}
R.~Rico-Martinez, K.~Krischer, I.~Kevrekidis, M.~Kube, and J.~Hudson,
  ``Discrete-vs. continuous-time nonlinear signal processing of cu
  electrodissolution data,'' \emph{Chemical Engineering Communications}, vol.
  118, no.~1, pp. 25--48, 1992.

\bibitem{livingston1993knot}
C.~Livingston, \emph{Knot theory}.\hskip 1em plus 0.5em minus 0.4em\relax
  Cambridge University Press, 1993, vol.~24.

\bibitem{armstrong2013basic}
M.~A. Armstrong, \emph{Basic topology}.\hskip 1em plus 0.5em minus 0.4em\relax
  Springer Science \& Business Media, 2013.

\bibitem{ranicki2013high}
A.~Ranicki, \emph{High-dimensional knot theory: Algebraic surgery in
  codimension 2}.\hskip 1em plus 0.5em minus 0.4em\relax Springer Science \&
  Business Media, 2013.

\bibitem{anderson1996comparison}
J.~Anderson, I.~Kevrekidis, and R.~Rico-Martinez, ``A comparison of recurrent
  training algorithms for time series analysis and system identification,''
  \emph{Computers \& chemical engineering}, vol.~20, pp. S751--S756, 1996.

\bibitem{VLACHAS2020191}
P.~Vlachas, J.~Pathak, B.~Hunt, T.~Sapsis, M.~Girvan, E.~Ott, and
  P.~Koumoutsakos, ``Backpropagation algorithms and reservoir computing in
  recurrent neural networks for the forecasting of complex spatiotemporal
  dynamics,'' \emph{Neural Networks}, vol. 126, pp. 191 -- 217, 2020.

\bibitem{ladijzenskaia1969mathematical}
O.~Ladijzenskaia, ``The mathematical theory of viscous incompressible fluid,''
  \emph{Gordon and}, vol. 667, 1969.

\bibitem{adam2015}
D.~P. Kingma and J.~Ba, ``Adam: {A} method for stochastic optimization,'' in
  \emph{3rd International Conference on Learning Representations, {ICLR} 2015,
  San Diego, CA, USA, May 7-9, 2015, Conference Track Proceedings}, 2015.

\bibitem{jin2020unit}
P.~Jin, Y.~Tang, and A.~Zhu, ``Unit triangular factorization of the matrix
  symplectic group,'' \emph{SIAM Journal on Matrix Analysis and Applications},
  vol.~41, no.~4, pp. 1630--1650, 2020.

\bibitem{kramer1991nonlinear}
M.~A. Kramer, ``Nonlinear principal component analysis using autoassociative
  neural networks,'' \emph{AIChE journal}, vol.~37, no.~2, pp. 233--243, 1991.

\end{thebibliography}
\end{document}